\documentclass[letterpaper, 10 pt, conference]{ieeeconf}
\IEEEoverridecommandlockouts    
\usepackage[dvipsnames]{xcolor}
\usepackage{lettrine}
\usepackage{balance}
\usepackage{graphicx}
\usepackage{amsmath}
\usepackage{amssymb}
\newcommand{\vect}[1]{\mathbf{#1}}
\newcommand{\mat}[1]{\mathbf{#1}}

\newcommand{\diffs}[3]{\frac{\partial^2 #1}{
\ifx#2#3 
\partial #2^2
\else
\partial #2 \partial #3
\fi
}}
\DeclareMathOperator*{\argmin}{arg\,min}

\newcommand{\av}{\vect{a}}
\newcommand{\bv}{\vect{b}}

\newcommand{\ev}{\vect{e}}

\newcommand{\fv}{\vect{f}}
\newcommand{\gv}{\vect{g}}

\newcommand{\hv}{\vect{h}}
\newcommand{\kv}{\vect{k}}

\newcommand{\nv}{\vect{n}}

\newcommand{\rhov}{\mathbf{\rho}}

\newcommand{\rv}{{\vect{r}}}

\newcommand{\uv}{\vect{u}}
\newcommand{\vv}{\vect{v}}

\newcommand{\xv}{\vect{x}}

\newcommand{\yv}{\vect{y}}

\newcommand{\Am}{\mat{A}}

\newcommand{\Gm}{\mat{G}}

\newcommand{\Lm}{\mat{L}}

\newcommand{\Wm}{\mat{W}}

\usepackage{svg}
\usepackage{float}
\usepackage{hyperref}
\usepackage{academicons}
\usepackage{xcolor}

\usepackage{amsthm}
\theoremstyle{plain}% default

\newtheorem{thm}{Theorem}
\newtheorem{lem}{Lemma}
\newtheorem{assumpt}{Assumption}
\newtheorem{problem}{Problem}

\usepackage{comment}
\usepackage[font=small]{caption}
\usepackage{subcaption}
\usepackage{calc}
\newtheorem{rem}{\textbf{Remark}}[section]
\usepackage{mathtools} 
\title{\LARGE \bf
Unified Feedback Linearization for Nonlinear Systems with\\ Dexterous and Energy-Saving Modes}
\usepackage{orcidlink}

\author{Mirko Mizzoni$^1$\orcidlink{0009-0006-2165-3475}, Pieter van Goor$^1$\orcidlink{0000-0003-4391-7014}, and Antonio Franchi$^{1,2}$\orcidlink{0000-0002-5670-1282}
\thanks{$^1$Robotics and Mechatronics group, Faculty of Electrical Engineering,  Mathematics, and Computer Science (EEMCS), University of Twente, 7500 AE Enschede, The Netherlands. {\footnotesize \tt m.mizzoni@utwente.nl}, {\footnotesize \tt p.c.h.vangoor@utwente.nl}, {\footnotesize \tt a.franchi@utwente.nl}}
\thanks{$^2$Department of Computer, Control and Management Engineering, Sapienza University of Rome, 00185 Rome, Italy, {\footnotesize \tt antonio.franchi@uniroma1.it}}\thanks{This work was partially funded by the Horizon Europe research agreement no. 101120732 (AUTOASSESS).}}

\begin{document}

\maketitle
\thispagestyle{empty}
\pagestyle{empty}

% takeawya lesson
% abstract >-> statement of a thereom 
% the paper will represent the proof 
% of my statement

\begin{abstract}
Systems with a high number of inputs compared to the degrees of freedom (e.g. a mobile robot with Mecanum wheels) often have a minimal set of \textit{energy-efficient} inputs needed to achieve a \textit{main task} (e.g. position tracking) and a set of \text{energy-intense} inputs needed to achieve an additional \textit{auxiliary task} (e.g. orientation tracking). 
This letter presents a unified control scheme, derived through feedback linearization, that can switch between two modes: 
an \textit{energy-saving mode}, which tracks the main task using only the energy-efficient inputs while forcing the \textit{energy-intense} inputs to zero, and a \textit{dexterous mode}, which also uses the energy-intense inputs to track the auxiliary task as needed.
The proposed control guarantees the exponential tracking of the main task and that the dynamics associated with the main task evolve independently of the a priori unknown switching signal.
When the control is operating in dexterous mode, the exponential tracking of the auxiliary task is also guaranteed.
Numerical simulations on an omnidirectional Mecanum wheel robot validate the effectiveness of the proposed approach and demonstrate the effect of the switching signal on the exponential tracking behavior of the main and auxiliary tasks.
\end{abstract}

\section{Introduction}
In many mechatronics systems, for example robotic or transportation systems involving high-dimensional and omnidirectional motion~\cite{TAHERI2020103958}, the full control input is not always necessary to complete a task. 
\begin{figure}[t]
    \centering
\includegraphics[trim={3.9cm 3.6cm 2.8cm 2.8cm},clip,scale=0.35]{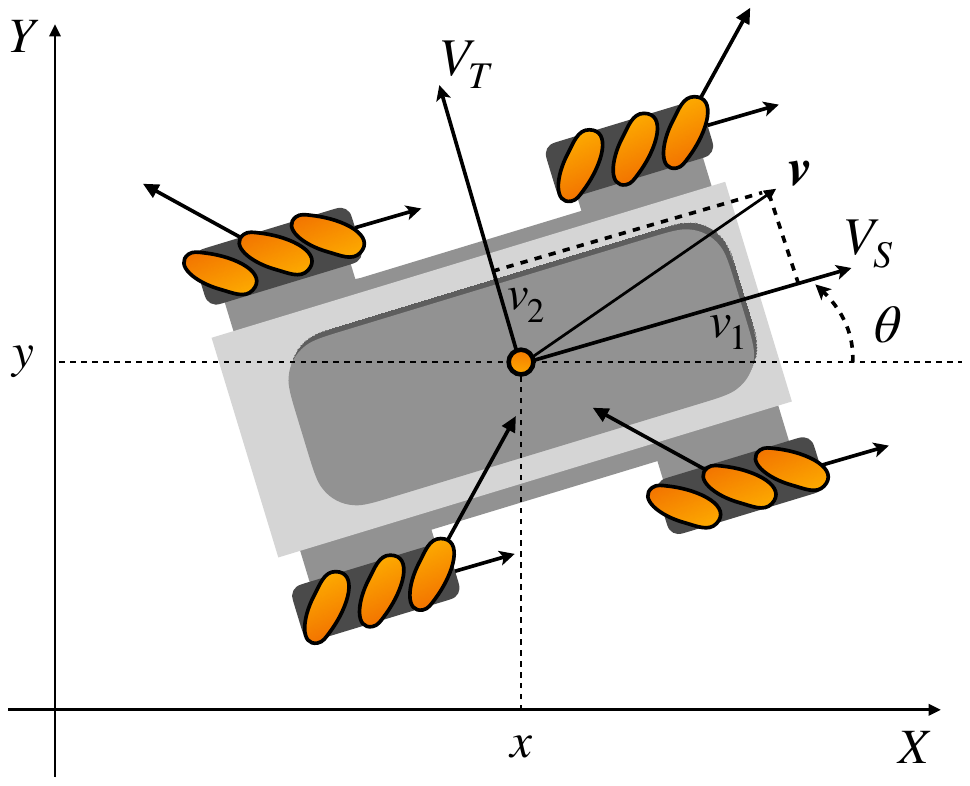}
    \caption{A four Mecanum wheels omnidirectional vehicle. This system is capable of lateral movement (dexterous mode), but this comes at the cost of higher power consumption compared to the more efficient forward-only motion (energy-saving mode). The application of the generic nonlinear controller introduced in this work to this system ensures independent and decoupled exponential stabilization of the position output, along with either the orientation (in dexterous mode) or the lateral speed (in energy-saving mode), depending on the a priori unknown operation mode selected by an external source.}
    \label{fig:model}
\end{figure} 
For instance, Mecanum wheel robots (see Fig.~\ref{fig:model}) are renowned for their omnidirectional mobility, allowing seamless movement in all directions on a flat surface. 
Nevertheless, the sideways movement is so energy-demanding that sometimes it is preferred to avoid or limit it, especially for tasks in which it is not strictly required. 
Another example is flying morphing platforms, such as those described in~\cite{aboudorra2023omnimorph,Ryll2022FASTHex}, which allow additional lateral forces that can be deactivated to conserve energy during point-to-point translation tasks.
This redundancy allows for selective input constraints that can optimize energy consumption by deactivating unnecessary degrees of freedom. 
Such strategies are especially relevant in systems where task requirements are confined to a lower-dimensional subspace, as observed in robotic manipulators and mobile platforms.
The principle of exploiting redundancy for efficiency ~\cite{siciliano_redound,10008952,deluca_tp} is inspired by biological motor control systems, which prioritize tasks-relevant actions while minimizing unnecessary effort.
Methods such as null space projection
~\cite{7097068,zhang2021improved,9326869} have demonstrated the feasibility of reducing actuator usage without compromising task performance.
For example, in \cite{7017672}, auxiliary tasks were achieved by prescribing motion in the null space of the quadrotor-arm Jacobian while preserving the primary task.
Nevertheless, these methods are tailored to the specific structure of the systems and are not general or easy to extend.
Moreover, the auxiliary tasks are not typically perfectly achieved since they are framed as the minimization of the argument of a given cost function, without the guarantee that this minimum is achieved.
Nonlinear Model Predictive Control (nMPC) offers a powerful alternative~\cite{ref_mpc}, enabling the prioritization of specific inputs through a carefully designed energy cost function ~\cite{2020l-BicMazFarCarFra,10591286}. While it has shown strong practical performance, nMPC is inherently susceptible to local minima, which can compromise system robustness.
Additionally, the presence of conflicting objectives may introduce instability, further limiting its reliability ~\cite{He2015OnSO}.

The present work addresses a class of system that are suitable to be controlled in two possible modes: \textit{energy-saving mode} and \textit{dexterous mode}.
These are system which have a primary \textit{main task} that has to be fulfilled at any time, and in parallel to that, an exogenous system (e.g., a human operator) decides at any time if (1) a secondary \textit{auxiliary task} must be also executed (\textit{dexterous mode}) or (2) a minimum number of control inputs should be used (i.e., have to be non-zero) in order to  minimize energy consumption (\textit{energy-saving mode}).

The main contribution is the development of a generic control framework for nonlinear systems with \textit{theoretical guarantees}.
This framework employs a mode-switching signal to transition between an \textit{energy-saving mode}, where unnecessary inputs are disabled to prioritize the \textit{main task}, and a \textit{dexterous mode}, where the remaining inputs are utilized to address the \textit{auxiliary task}. 
Crucially, we demonstrate that the satisfaction of the \textit{main task} is preserved throughout the process—before, during, and after the switching event.

The paper is organized as follows. Section~\ref{sec:MotExmp} introduces a motivating example. 
Section~\ref{subs:preliminaries}  presents the required notions for describing our framework.
In Sections~\ref{subs:sect_4} and~\ref{sect:method}, we formulate the problem and present the proposed solution.
The paper concludes with a continuation of the motivating example and corresponding simulations.

%\textit{Notations}: In what follows, $c_\theta$ and $s_\theta$ are used as a shorthand for $\cos\theta$ and $\sin\theta$, respectively.

\section{Notation}
We use the shorthand notation \( c_\theta \) for \( \cos\theta \) and \( s_\theta \) for \( \sin\theta \). 
We denote the identity matrix of size \( n \times n \) by \( \mathbf{I}_n \).
Given a vector $\av_m$, we denote its  $i$-th entry by  $a_{m,i}$.

\section{Motivating Example}\label{sec:MotExmp}

As a simple motivating example, consider a four Mecanum wheel omnidirectional vehicle (see Fig.~\ref{fig:model}) whose configuration is given
by the variables $(x,y,\theta)\in \mathbb{R}^2\times \mathcal{S}^1$ representing the
coordinates of the center point and the orientation angle of the vehicle in an inertial
frame, respectively.

The system has three inputs: the sagittal ($v_1$) and transversal ($v_2$) velocities, and the turning rate of the vehicle ($v_3$). 
These inputs are mapped one-to-one to a three dimensional subspace of the four-wheel angular velocities through a simple kinematic relation that is omitted here for simplicity, see~\cite{Muir1990} for the details. 
The dynamical model is given by
\begin{equation}
\Sigma_{\text{4W}}: \left\{ 
\begin{split}
    \dot{x}&=v_1c_\theta -v_2s_\theta,  \\
    \dot{y}&=v_1s_\theta + v_2 c_\theta, \\
     \dot{\theta}&=v_3,
\end{split}
\right.
\label{eq:4w}
\end{equation}
This type of vehicle is more dexterous than non-holonomic systems, like differential drive robots and car-like vehicles, thanks to its ability to also move laterally.
This makes it an ideal solution for precise maneuvering in restricted spaces and motivates its wide use in logistics~\cite{logistics4W,logistics4W_2}.
The downside, as has been shown in~\cite{zhou2016experimental}, is that this vehicle consumes twice the energy when moving in the transversal direction at a given speed as when moving in the sagittal direction at the same speed.

Our goal is obtain a single control scheme that allows the vehicle to seamlessly switch between a \emph{dexterous mode}, where the transversal velocity input $v_2$ is used and each degree of freedom of the configuration $(x,y,\theta)$ is controlled independently, and an \emph{energy-saving mode}, where only the position variables $(x,y)$ are controlled independently (as in a non-holonomic vehicle) and $v_2$ is nullified to minimize the energy consumed. 
We treat the decision about when to switch between the two modes as an a priori unknown exogenous signal $\sigma\in\{0,1\}$, which could be provided by, e.g., a high level operator, depending on the task at hand.
For example, dexterous mode might be requested for accurate maneuvering and manipulation in narrow spaces, and energy saving mode for long distance transportation in large spaces. 

Model identification followed by feedback linearization control schemes are both effective and widely used in practice in industrial robotic systems due to their precise knowledge of system model parameters. 
A naive approach to obtain the sought feedback linearizing controller would be to switch between two different control schemes, one designed  for the dexterous mode and one designed for the energy-saving mode.
In dexterous mode the vector relative degree of the outputs $(x,y,\theta)$ is $\{1,1,1\}$ and a suitable controller is given by
\begin{equation}
\begin{split}
   \begin{pmatrix}
v_1\\
v_2\\
v_3
\end{pmatrix} = \begin{pmatrix}
c_\theta & -s_\theta & 0\\
s_\theta & c_\theta & 0\\
0 & 0 & 1
\end{pmatrix}^{-1}\begin{pmatrix}
    \dot{x}^d + k_{p1}(x^d-x) \\
    \dot{y}^d + k_{p2}(y^d-y)  \\
    \dot{\theta}^d + k_{p3}(\theta^d-\theta)
\end{pmatrix}\ ,
\end{split}
\label{eq:contrUFull}
\end{equation}
where $k_{p1},k_{p2},k_{p3} > 0$ are chosen gains.
In energy-saving mode, when $v_2=0$, one can add a dynamic extension of the input $v_1$ which leads to
a vector relative degree of the outputs $(x,y)$ equal to $\{2,2\}$, and a suitable controller is given by

\begin{equation}
\left\{
     \begin{split}
        v_2 &= 0 , \quad v_1(0)\neq 0,\\
        \begin{pmatrix}
        \dot{v}_1\\
        v_3
        \end{pmatrix} 
        &=
             \begin{pmatrix}
        c_\theta & -v_1 s_\theta  \\
        s_\theta & v_1 c_\theta
        \end{pmatrix}^{-1}
                \begin{pmatrix}
        w_1\\
        w_2
        \end{pmatrix}\\ 
        w_1&=  \ddot{x}^d + k_{p1}(x^d - x) + k_{d1}(\dot{x}^d -\dot{x}),\\
        w_2&=\ddot{y}^d + k_{p2}(y^d - y) + k_{d2}(\dot{y}^d -\dot{y}),
    \end{split}
    \right.
    \label{eq:contrUUnd}
\end{equation}
where $k_{p1},k_{p2},k_{d1},k_{d2} > 0$ are chosen gains~\cite{DELUCA2000687}.

The first issue with simply switching between two unrelated control schemes is in the behavior of the output error dynamics.
In both cases, the position $(x,y)$ is controlled as an output of the system, but the relative degree is changed.
This means that it may not be possible to switch between the controllers without inducing a transient behavior in the error dynamics, and in the worst case, it may not even be possible to guarantee stability or convergence when switching occurs.
Another issue in practice is that when $\sigma$ switches, for example, from $1$ to $0$, the transversal velocity $v_2$ is instantaneously set to zero and therefore it is subject to a discontinuous jump, which is undesirable in practice due to the fact that no real vehicle can change its velocity instantaneously due to its inertia. 
Attempting to abruptly halt the vehicle's lateral motion would likely result in slippage or, worse, damage to components of the actual vehicle.
To address this challenge, our goal is to develop a unified controller which receives as input a desired trajectory of the position, a desired orientation angle trajectory, and an exogenous switching signal $\sigma$ which dictates the mode of our controller. 
We require that:
\begin{enumerate}
    \item The position tracking error is exponentially stable and evolves independently of the value of $\sigma$.
    \item The orientation tracking error is exponentially stable whenever $\sigma = 1$ (dexterous mode).
    \item The transversal velocity is brought exponentially to zero whenever $\sigma = 0$ (energy-saving mode) without discontinuity. 
\end{enumerate}

\section{Preliminaries}\label{subs:preliminaries}

For a comprehensive introduction to feedback linearization, the interested reader is referred to \cite{Isidori1995}.

Consider a multivariable nonlinear system
    \begin{equation}
    \left\{
\begin{split}
     \dot{\xv}&= \fv(\xv)+\Gm(\xv)\uv,\\
    \yv &= \hv(\xv),
    \label{eq:sys}
    \end{split}
    \right.
    \end{equation}
where \mbox{$\xv\in \mathbb{R}^n$} is the state, the input matrix is \mbox{$\Gm(\xv)=\begin{bmatrix}
    \gv_1(\xv)  \;\cdots \; \gv_{p}(\xv)
\end{bmatrix}\in \mathbb{R}^{n \times p}$}, $\fv(\xv)$,  $\gv_1(\xv), \ldots, \gv_{p}(\xv)$ are smooth vector fields, and $\hv(\xv)=\begin{bmatrix} h_1(\xv)  \cdots h_{p}(\xv)\end{bmatrix}^\top$ is a smooth function defined on an open set of $\mathbb{R}^n$.
The  system (\ref{eq:sys}) is said to have \emph{(vector) relative degree} $\rv = \{
    r_1, \ldots, r_{p}
\}$ at a point $\xv^\circ$ w.r.t. the input-output pair  $(\uv,\yv)$ if  
\begin{flalign}
&\text{\textrm{(i)}}    & L_{\gv_j}L^{k}_{\fv} h_i(\xv) &=0,&
\end{flalign}
for all $1\leq j \leq p$, for all $k\leq r_i-1$, for all $1\leq i \leq p$ and for all $\xv$ in a neighborhood of $\xv^\circ$, and\\
\textrm{(ii)}\;\;  the $p\times p$ matrix 
   \begin{align}
    \Am(\xv) &:= 
         \begin{pmatrix}
            L_{\gv_1}L^{r_1-1}_{\fv}h_1(\xv) & \cdots&  L_{\gv_{p}}L^{r_1-1}_{\fv}h_1(\xv) \\ 
            L_{\gv_1}L^{r_2-1}_{\fv}h_2(\xv) & \cdots&  L_{\gv_{p}}L^{r_2-1}_{\fv}h_2(\xv) \\  
            \vdots & & \vdots \\
            L_{\gv_1}L^{r_{p}-1}_{\fv}h_{p}(\xv) & \cdots&  L_{\gv_{p}}L^{r_{p}-1}_{\fv}h_{p}(\xv) 
        \end{pmatrix}
    \label{eq:intbmatrix}
\end{align}  
is nonsingular at $\xv = \xv^\circ$. 
The output array at the $\rv$-th derivative may then be written as an affine system of the form
\begin{equation}
{\yv}^{(\rv)} :=\begin{bmatrix}
    y_1^{(r_1)} \;\cdots\;  y_{p}^{(r_{p})}
\end{bmatrix}^\top =\bv(\xv)+\Am(\xv)\uv,
\label{eq:y_r_now}
\end{equation}
with 
\begin{equation}
 \bv(\xv):=\begin{bmatrix}
L_{\boldsymbol{f}}^{(r_1)}{h_1(\xv)} \; \cdots \; 
L_{\boldsymbol{f}}^{(r_{p})}{h_{p}(\xv)}
\end{bmatrix}^\top.
\label{eq:b}
\end{equation}

Suppose the system \eqref{eq:sys} has some \emph{(vector) relative degree} $\rv:=\{r_1,\ldots,r_p\}$ at $\xv^\circ$ and that the matrix $\Gm(\xv^\circ)$ has rank $p$ in a  neighborhood $\mathcal{U}$ of $\xv^\circ$. Suppose also that  \mbox{$r_1+r_2+\ldots+r_p=n$}, and choose the  control input to be $$\uv = \Am^{-1}(\xv)[-\bv(\xv)+\vv],
$$
where $\vv \in \mathbb{R}^{p}$ can be assigned freely and $\Am(\xv), \bv(\xv)$ are defined as in~(\ref{eq:intbmatrix}) and~(\ref{eq:b}).
Then the output dynamics \eqref{eq:y_r_now} become
$$
\yv^{(\rv)} = \vv.
$$
We refer to \(\yv\) as a \textit{linearizing output array}, which possesses the property that the entire state and input of the system can be expressed in terms of \(\yv\) and its time derivatives.

\section{Problem Formulation}\label{subs:sect_4}

In this section, we generalise our motivating example and state our formal problem definition.

Consider a system $\Sigma$ of the form
\begin{equation}
\Sigma:\left\{
    \begin{split}
\dot{\xv}&=\fv(\xv)+\Gm(\xv)\uv, \\
\yv_1&=\hv_1(\xv_1),\\
\yv_2&=\hv_2(\xv_1),\\
\yv_3&=\hv_3(\xv_2),
\label{eq:Sigma12}
\end{split}
\right.
\end{equation} with $\xv=\begin{bmatrix}
    \xv_1^\top & \xv_2^\top
\end{bmatrix}^\top\in\mathbb{R}^n$, where $\uv\in \mathbb{R}^{p} $ is the control input, and $\yv_i \in \mathbb{R}^{p_i}$ are the output arrays, with $i=1,2,3$, $p_1+p_2=p$, and $p_3=p_2$. We
refer to  $\xv_2$ as the state of the \emph{energy-intense} actuation part and we assume that $\Sigma$ possesses the property that $\xv_2 = \boldsymbol{0}$ whenever $\yv_3 = \boldsymbol{0}$, $\dot{\yv}_3 = \boldsymbol{0}$, and all higher-order derivatives of $\yv_3$ up to a finite order are $\boldsymbol{0}$.

Denote with $\yv = \begin{bmatrix}
    \yv_1^\top & \yv_2^\top
\end{bmatrix}^\top$ and with $\bar{\yv} = \begin{bmatrix}
    \yv_1^\top & \yv_3^\top
\end{bmatrix}^\top$.
%We assume
Consider a desired output  trajectory \mbox{$\yv^{d}=\begin{bmatrix}
    (\yv^{d}_1)^\top &  (\yv^{d}_2)^\top & \boldsymbol{0}^\top
\end{bmatrix}^\top$} where \mbox{$\yv^{d}_i: [0,\infty) \to \mathbb{R}^{p_i}$}, $i=1,2$,  a signal $\sigma: [0,\infty) \to \{0,1\}$, and define the error $
\ev = \begin{bmatrix}
    \ev_1^\top &
    \ev_2^\top &
    \ev_3^\top
\end{bmatrix} := \begin{bmatrix}
    \yv^d_1 - \yv_1\\
    \yv^d_2 - \yv_2 \\
    -\yv_3
\end{bmatrix}$
with respect to the desired output trajectory.

\begin{problem}\label{prob:main_problem}
Design a smooth controller which, under suitable assumptions, obtains the following output specifications: 
    \begin{enumerate}
\item The dynamics of the error $\ev_1$ are exponentially stable\label{c:1} and independent of the value of $\sigma$.
    \item When $\sigma = 1$, the dynamics of the error 
    $\ev_2$ are exponentially stable\label{c:2}.
    \item When $\sigma = 0$, the dynamics of the error 
    $\ev_3$ are exponentially stable\label{c:3}.
\end{enumerate}
\end{problem}
We make the following assumptions.
\begin{assumpt}\label{asmpt_1}
    The system  \eqref{eq:Sigma12}  has a \emph{(vector)} relative degree $\rv  =\{\rhov_1,\rhov _2\}$ with $\rhov_1=\{r_1,\ldots,r_{p_1}\}, \rhov_2=\{r_{p_1+1},\ldots,r_{p}\}$  w.r.t. the pair ($\uv,\yv)$ and $\bar{\rv}=\{\rhov_1,\rhov_3\}$ with $\rhov_3=\{\bar{r}_{p_1+1},\ldots,\bar{r}_{p}\}$  w.r.t. the pair ($\uv,\bar{\yv}$) at $\xv^\circ$. 
\end{assumpt}
\begin{assumpt}\label{asmpt_2} 
The system \eqref{eq:Sigma12}
  is such that $\sum_{i=1}^pr_i=n$ and $\sum_{i=1}^{p_1}r_i+\sum_{i=p_1+1}^p\bar{r}_i =n$.
\end{assumpt}

\section{Proposed method}\label{sect:method}
\begin{lem}\label{lem:1}
Let $\sigma:[0,\infty)\to \{0,1\}$ be an exogenous switching signal, and let $\vv_1$, $\vv_2$, and $\vv_3$ be arbitrary assignable virtual inputs. 
If Assumptions~\ref{asmpt_1} and~\ref{asmpt_2} hold, then 
there exists a feedback controller $\uv = \kv_{\sigma}(\xv,\vv_1, \vv_2,\vv_3$) such that
\begin{equation}
\left\{
\begin{split}
    \yv_1^{(\rv')}&=\vv_1 \quad \text{at any time},\\
    \yv_2^{(\rv'')}&=\vv_2  \quad \text{when $\sigma=0$},\\
    {\yv}_3^{(\bar{\rv}'')}&=\vv_3 \quad \text{when $\sigma=1$},
    \end{split}
    \right.
\label{eq:output_dyn_tilde}
\end{equation}
and the system has trivial internal dynamics.
\label{lem_1}
\end{lem}

\begin{proof}
We denote with $\Am(\xv)$ the resulting interaction matrix obtained at the $\rv$-th differentiation of the output vector $\yv$ and with  $\bar{\Am}(\xv)$ the resulting interaction matrix obtained at the $\bar{\rv}$-th differentiation of the output vector $\bar{\yv}$. For clarity, we partition the two matrices into submatrices i.e.
\begin{align*}
 \Am(\xv)=\left[\begin{smallmatrix}
 \Am_{11}(\xv) & \Am_{12}(\xv)\\
   \Am_{21}(\xv) & \Am_{22}(\xv)\\
\end{smallmatrix}\right], 
&&
\bar{\Am}(\xv)=\left[\begin{smallmatrix}
 \bar{\Am}_{11}(\xv) & \bar{\Am}_{12}(\xv)\\
   \bar{\Am}_{21}(\xv) & \bar{\Am}_{22}(\xv)\\
\end{smallmatrix}\right], 
\end{align*}
in which $\bar{\Am}_{11}(\xv)=\Am_{11}(\xv)$ and   $\bar{\Am}_{12}(\xv)=\Am_{12}(\xv)$ since the two vectors $\yv$ and $\bar{\yv}$ share the same first entries $\yv_1$.
We know by assumption that both $\Am(\xv)$ and $\bar{\Am}(\xv)$ are nonsingular matrices.
Consider now the output \begin{equation}
{\yv}_{\sigma}= \begin{bmatrix}
 \yv_1 ^\top &
 \sigma \yv_2^\top + (1-\sigma)\yv_3^\top
\end{bmatrix}^\top.
\label{y_tilde}
\end{equation}
The corresponding interaction matrix obtained when considering the pair ($\uv,{\yv}_{\sigma}$) has the structure
\begin{equation}
{\Am}_{\sigma}(\xv) = \left[\begin{smallmatrix}
 \Am_{11}(\xv) & \Am_{12}(\xv) \\
 \sigma \Am_{21}(\xv)+(1-\sigma) \bar{\Am}_{21}(\xv) &   \sigma \Am_{22}(\xv)+(1-\sigma) \bar{\Am}_{22}(\xv)
\end{smallmatrix}\right].
\label{eq:Atilde}
\end{equation}
The determinant of ${\Am}_{\sigma}(\xv)$ is 
\begin{equation}
    \sigma \operatorname{det}\Am(\xv) + (1-\sigma) \operatorname{det}\bar{\Am}(\xv).
\end{equation}
Hence, independently of the value of $\sigma\in \{0,1\}$, the matrix ${\Am}_{\sigma}(\xv)$ is invertible in a neighborhood of $\xv^\circ$. 
Moreover, from Assumption \ref{asmpt_2}, the sum of the relative degree is always equal to the dimension of the state space $n$. 
It follows that, at the ${\rv}_{\sigma}$-th differentiation of the output ${\yv}_{\sigma}$, (where ${\rv}_{\sigma}=\rv $ when $\sigma=1$ and ${\rv}_{\sigma}=\bar{\rv}$ when $\sigma=0$) we get
\begin{equation}
{\yv}^{({\rv}_{\sigma})}_{\sigma} = \bv(\xv) + {\Am}_{\sigma}(\xv)\uv.
\end{equation}
By choosing the feedback linearizing controller
\begin{equation}
    \uv = {\Am}_{\sigma}(\xv)^{-1}[-\bv(\xv)+{\vv}_{\sigma}],
    \label{eq:control}
\end{equation}
with \begin{equation}
    {\vv}_{\sigma}=\begin{bmatrix}
    \vv_1^\top & \sigma \vv_2^\top+(1-\sigma ){\vv}_3^\top
\end{bmatrix}^\top,
\label{eq:v}
\end{equation}
the dynamics of the switching output ${\yv}_{\sigma}$ is given by~\eqref{eq:output_dyn_tilde}.
The internal dynamics are trivial since $~{\sum_{i}^n{r}_{\sigma,i}=n}$.
\end{proof}

The overall control architecture is shown in Figure~\ref{fig:lemma1}.
\begin{figure}[t]
    \centering
   \includegraphics[trim={3.5cm 2cm 1.3cm 2cm},clip,scale=0.3]{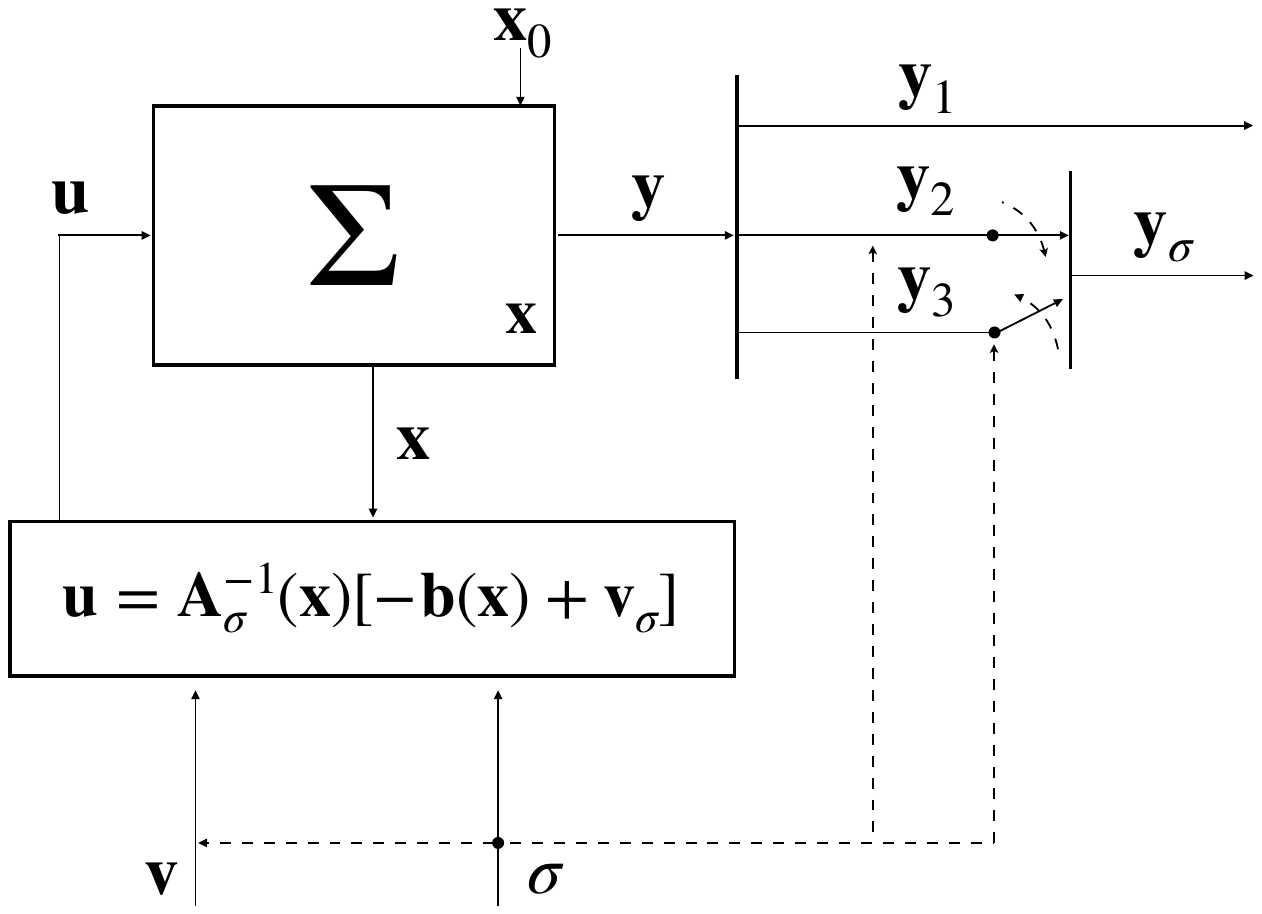}
    \caption{The Proposed Control Architecture.} 
    \label{fig:lemma1}
\end{figure}

\begin{thm}\label{thm:1}

    If  Assumptions~\ref{asmpt_1} and \ref{asmpt_2},  hold, then  a  solution of Problem~\ref{prob:main_problem}
is given by  the controller~\eqref{eq:control} with $\vv_1,\vv_2,\vv_3$ in~\eqref{eq:v} given by 
\begin{equation}
    \vv_j = \yv_j^{d,(\rhov_j)}+ \sum^{\rhov_{j}-1}_{i=1}\Lm^i_{j} ({\yv}^{d,(i)}_j-{\yv}^{(i)}_j), \quad j=1,2,3
    \label{eq:control2}
\end{equation} where  the $\Lm^i_j$ matrices are chosen arbitrarily, subject only to the constraint that substituting~\eqref{eq:control2} into~\eqref{eq:output_dyn_tilde} results in a linear output error dynamics that is exponentially stable.
\end{thm}

\begin{proof}
    If Assumptions~\ref{asmpt_1} and~\ref{asmpt_2} hold, then Lemma \ref{lem:1} applies. Hence, there exists a feedback linearizing controller \eqref{eq:control}  such that~\eqref{eq:output_dyn_tilde} holds. At this point, if we replace $\vv$ with \eqref{eq:v} then the error dynamics are given by 
    \begin{equation}
\left\{
\begin{split}
\ev_1^{(\rhov_1)} + \sum^{\rhov_{1}-1}_{i=1} \Lm^i_{1} \ev^{(i)}_1 &= \boldsymbol{0}, \quad \text{at any time},\\
\ev_2^{(\rhov_2)} + \sum^{\rhov_{2}-1}_{i=1} \Lm^i_{2} \ev^{(i)}_2 &= \boldsymbol{0}, \quad \text{when $\sigma=0$},\\
\ev_3^{(\rhov_3)} + \sum^{\rhov_{3}-1}_{i=1} \Lm^i_{3} \ev^{(i)}_3 &= \boldsymbol{0}, \quad \text{when $\sigma=1$}.
\end{split}
\right.
\end{equation}
    which solves the Problem \ref{prob:main_problem}.
\end{proof}
One possible choice of $\Lm^i_j$ is given by diagonal positive definite matrices.

\begin{rem}
    The closed loop system can be seen as a 
     switched system with externally forced switching \cite{Liberzon2003,survey_hybrid}
   of the form $ \dot{\xv}=\fv_{\sigma}(\xv,t)
    $ with output array ${\yv}_{\sigma}$ given in \eqref{y_tilde} and with switching signal $\sigma$.
\end{rem}

\section{Application to the Motivating Example}

The Mecanum wheel robot system \eqref{eq:4w} presented in Section \ref{sec:MotExmp} can be cast into the required form \eqref{eq:sys} by using a dynamic extension of the linear velocities $v_1, v_2$.
We thus define the state variables as $\xv_1 = \begin{bmatrix} x & y & \theta & v_1 \end{bmatrix}^\top$ and $\xv_2 = v_2$.
This extension incorporates the energy-intense input as part of the state, and is necessary so that the resulting system satisfies the assumptions \ref{asmpt_1} and \ref{asmpt_2}.

\subsection{Derivation of the controller equations}

When operating the platform in \textit{energy-saving mode}, the objective is to reduce the transversal velocity $v_2$ to zero, thereby minimizing energy consumption. 
Conversely, in \textit{dexterous mode}, the goal is to achieve trajectory tracking of the full configuration of the vehicle  (position and orientation).
To address these objectives, we define the output vector as $\yv = \begin{bmatrix} \yv_1^\top & y_2 & y_3 \end{bmatrix}^\top$, where $\yv_1 = \begin{bmatrix} x & y \end{bmatrix}^\top$, $y_2 = \theta$, and $y_3 = v_2$. The input vector is given by $\uv = \begin{bmatrix} u_1 & u_2 & u_3 \end{bmatrix}^\top$, with
    $\dot{v}_1=u_1$, 
    $\dot{v}_2=u_2$ 
 and $u_3 := v_3$.
The system exhibits a \emph{(vector) relative degree} $\rv = \{\rhov_1, \rhov_2\}$, where $\rhov_1 = \{2, 2\}$ and $\rhov_2 = 1$, relative to the pair $(\uv, \yv)$. Additionally, it exhibits a  \emph{(vector) relative degree}  $\bar{\rv} = \{\rhov_1, \rhov_3\}$, with $\rhov_3 = 1$, relative to the pair  $(\uv, \bar{\yv})$. Consequently, Assumptions~\ref{asmpt_1} and~\ref{asmpt_2} hold, allowing the application of Theorem~\ref{thm:1}.
Consider now the output ${\yv}_{\sigma}$ defined  in \eqref{y_tilde}. At the $\tilde{\rv}$-th differentation, the dynamics are given by $$
{\yv}^{({\rv}_{\sigma})}_{\sigma}={\Am}_{\sigma}(\xv)\uv 
$$ where the interaction matrix defined in~\eqref{eq:Atilde} is 
$${\Am}_{\sigma}(\xv) = \begin{pmatrix}
    c_\theta  & -s_\theta &-v_1s_\theta-v_2c_\theta   \\ s_\theta& c_\theta& v_1c_\theta-v_2s_\theta  \\
    0 & 1-\sigma & \sigma
\end{pmatrix}$$% 
and $\bv(\xv)=\boldsymbol{0}$. Then, using \eqref{eq:control}, the control law  becomes 
\begin{equation}
    \uv = {\Am}_{\sigma}(\xv)^{-1}\vv_{\sigma}
\label{eq:c_un}
\end{equation}
 with $\vv_{\sigma}$ as in \eqref{eq:v} and the $\vv_j$ as in \eqref{eq:control2} solves the Problem~\ref{prob:main_problem} for the four mecanum wheel vehicle.

\begin{figure}[t]
\centering
\includegraphics[width=0.54\columnwidth]{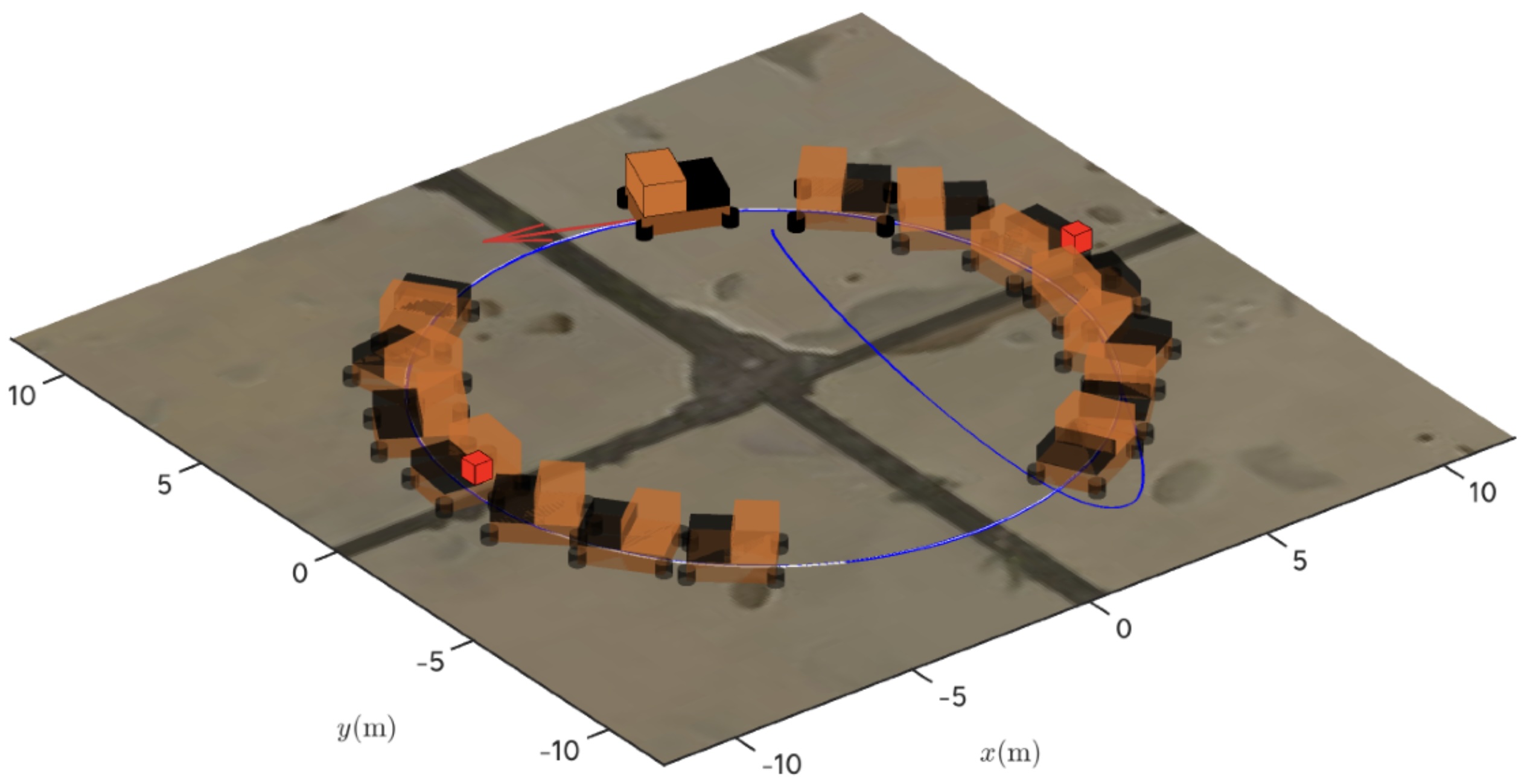}
\hfill
\includegraphics[width=0.44\columnwidth]{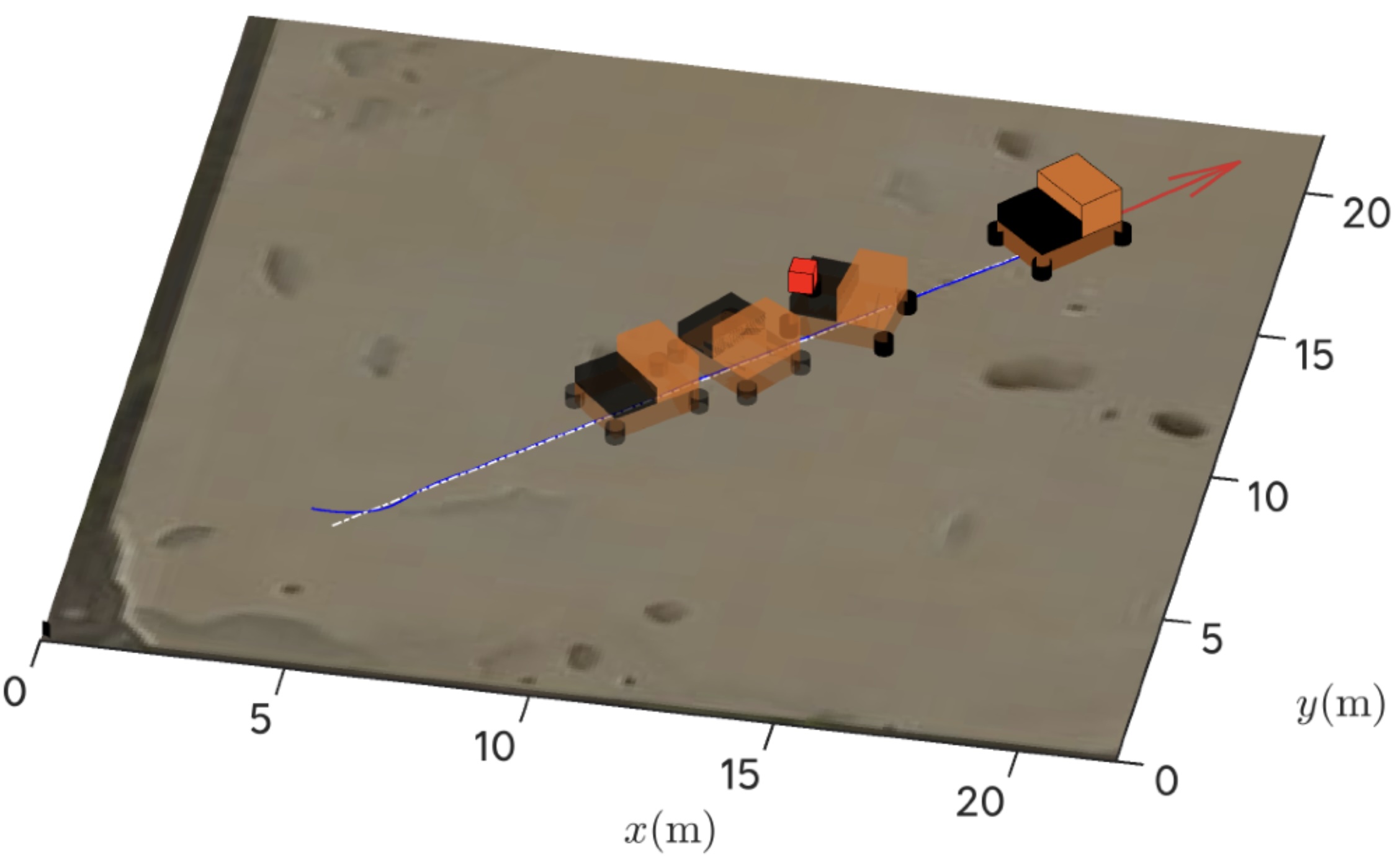}
      \caption{Stroboscopic highlights of two simulations.
      In \textbf{Simulation~1}  (left), the robot converges to and follows a circular trajectory.
      In \textbf{Simulation~2} (right), the robot converges to a straight-line trajectory  while carrying a load (depicted in orange) and avoiding hanging obstacles (shown in red).
      The robot operates in dexterity mode only when necessary (as determined by the switching signal $\sigma$), prioritizing energy-saving mode when far from obstacles.}
\label{fig:animation}
\end{figure}

\begin{rem}
The dynamic extension of the sagittal velocity leads to a singularity in the resulting controller.
Specifically, the decoupling matrix ${\Am}_{\sigma}(\xv)$ becomes singular when $\sigma = 0$ and $v_1 = 0$.
Such a singularity must be carefully addressed and avoided during trajectory planning, particularly when employing interpolation techniques. This can typically be achieved by appropriately selecting the initialization of the state $v_1$—an additional degree of freedom available in the design.
   
\end{rem}

\subsection{Numerical Simulations}
To demonstrate the capabilities of the proposed controller in a realistic scenario, we present two simulations where a four-Mecanum-wheel omnidirectional robot is tasked with transporting a load while following two different trajectories in the presence of obstacles (see Fig.~\ref{fig:animation}).

The first trajectory is a circular path defined by:
\[
\yv^d_1(t) = \begin{bmatrix}
r\sin(\omega t) &
-r\cos(\omega t)
\end{bmatrix}^\top, 
\]
where $ r = 8 \, \text{m} $ and $ \omega = 0.15 \, \text{rad/s} $. 

Two hanging obstacles are positioned along the trajectory. When the robot encounters these obstacles, a reorientation of $\pi/2 \, \text{rad}$ is required to avoid collisions between the transported load and the obstacles. The desired orientation for obstacle avoidance is defined as:
\[
y_2^d(t) = \omega t + \frac{\pi}{2}.
\]

For the majority of the trajectory, the switching signal $\sigma$ is set to zero, meaning that the robot operates in energy-saving mode. However, when passing beneath the obstacles, $\sigma$ switches to one, activating the dexterous mode. In this mode, the robot adjusts its orientation to avoid collisions while continuing to follow the position trajectory. 

\begin{figure}[t]
    \centering
    \includegraphics[trim={0.5cm 0cm 0cm 0cm},clip,scale=0.174]{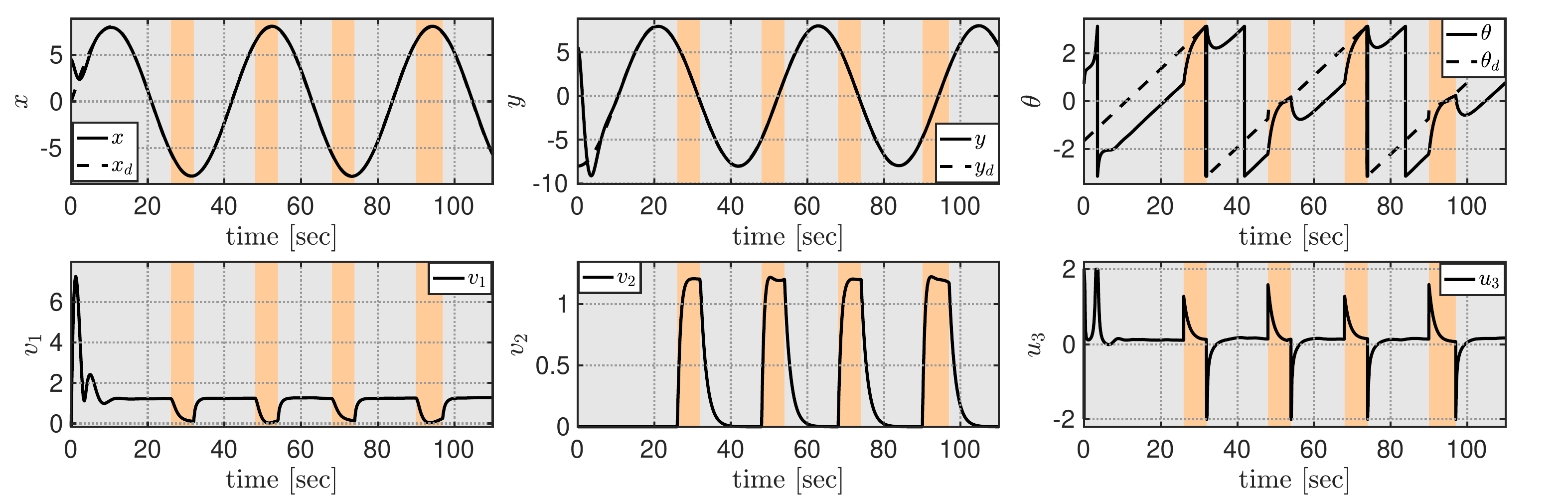}
    \caption{\textbf{Simulation~1.}  A circular input reference trajectory for the position of the CoM and a square form switching signal ${\sigma}$ are given to the control system. The gray areas correspond to $\sigma=0$ whereas the orange areas correspond to $\sigma=1$. The top row shows the output variables $\yv_1,\yv_2$, and the bottom row shows the sagittal velocity $v_1$, the third output $v_2$ and the control input $u_3$.  }
    \label{fig:circle}
\end{figure}

\begin{figure}[t]
    \centering
    \includegraphics[trim={0.6cm 0cm 0cm 0cm},clip,scale=0.174]{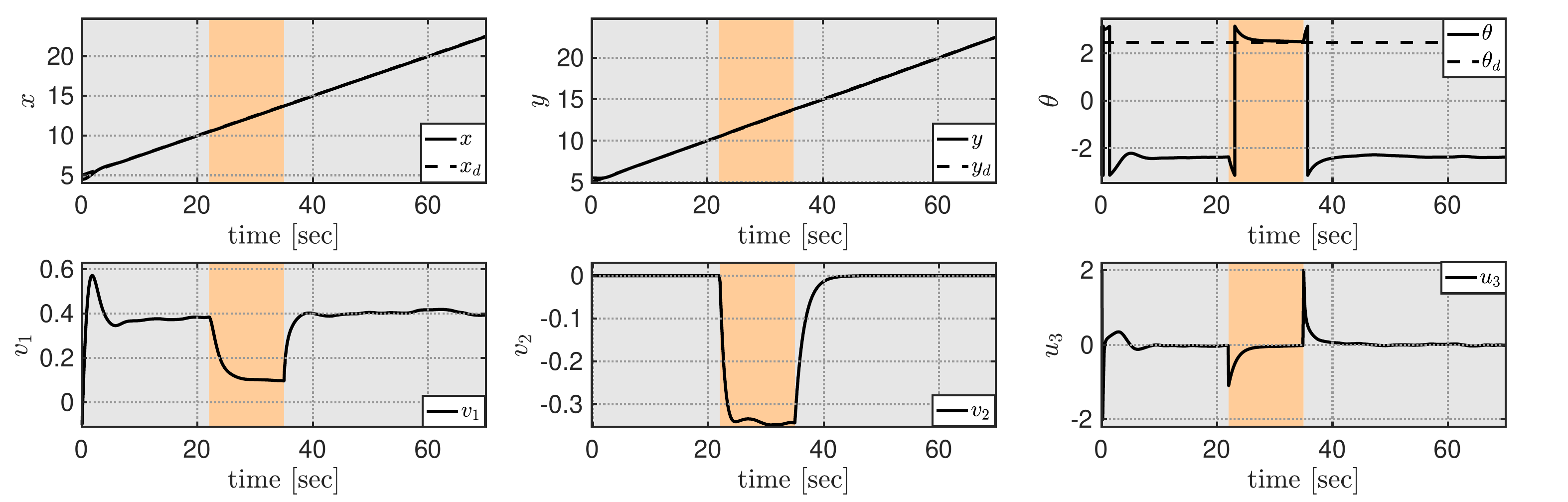}
\caption{\textbf{Simulation~2.} A ramp input reference trajectory with a square form switching signal $\sigma$ are given to the control system.
The gray areas correspond to $\sigma=0$ whereas the orange areas correspond to $\sigma=1$. The top row shows the output variables $\yv_1,\yv_2$, and the bottom row shows the sagittal velocity $v_1$, the third output $v_2$ and the control input $u_3$. }
    \label{fig:ramp} 
\end{figure}

Fig.~\ref{fig:circle} illustrates the simulation results. The gray areas correspond to $\sigma=0$ (energy-saving mode), while the orange areas represent $\sigma=1$ (dexterous mode).
The plots show the output variables $\yv_1$ and $\yv_2$ at the top, and the sagittal velocity $v_1$, the third output $v_2$, and the control input $u_3$ at the bottom.

The second trajectory is a straight line defined by:
\[
\yv^d_1(t) = \begin{bmatrix}
5 + \frac{t}{4} &
5 + \frac{t}{4}
\end{bmatrix}^\top.
\]

Similarly, a hanging obstacle is placed along this trajectory. To avoid a collision, the desired orientation is given by:
\[
y_2^d(t) = \frac{3\pi}{4}.
\]

The initial conditions for both simulations were set far from the trajectory to showcase transient behaviors. The gain matrices used were:
\[
\Lm^1_{1} = \mathbf{I}_2, \quad
\Lm^2_{1} = \mathbf{I}_2, \quad
L^1_2 = 0.75, \quad
L^1_3 = 0.65.
\]
This simulation also incorporated a low-pass filtered Gaussian noise $\nv \in \mathbb{R}^3$ in the actuation inputs to enhance realism. 
Precisely, to the control input \eqref{eq:c_un} we add the noise resulting from the solution of 
\begin{align}\label{eq:lowpass-noise}
    \dot{\nv}=-k\nv+\boldsymbol{\mu}, && \boldsymbol{\mu} \in \mathcal{N}(\boldsymbol{0}, q^2 \mathbf{I}_3),
\end{align}
where $q = 0.4$ and, $k=0.1$. 
Figure \ref{fig:noise} shows the components of the noise $\nv$ added to the actuators during Simulation~1.

%\begin{comment}
\begin{figure}[t]
    \centering
    \includegraphics[trim={0.04cm 0cm 0cm 0cm},clip,scale=0.171]{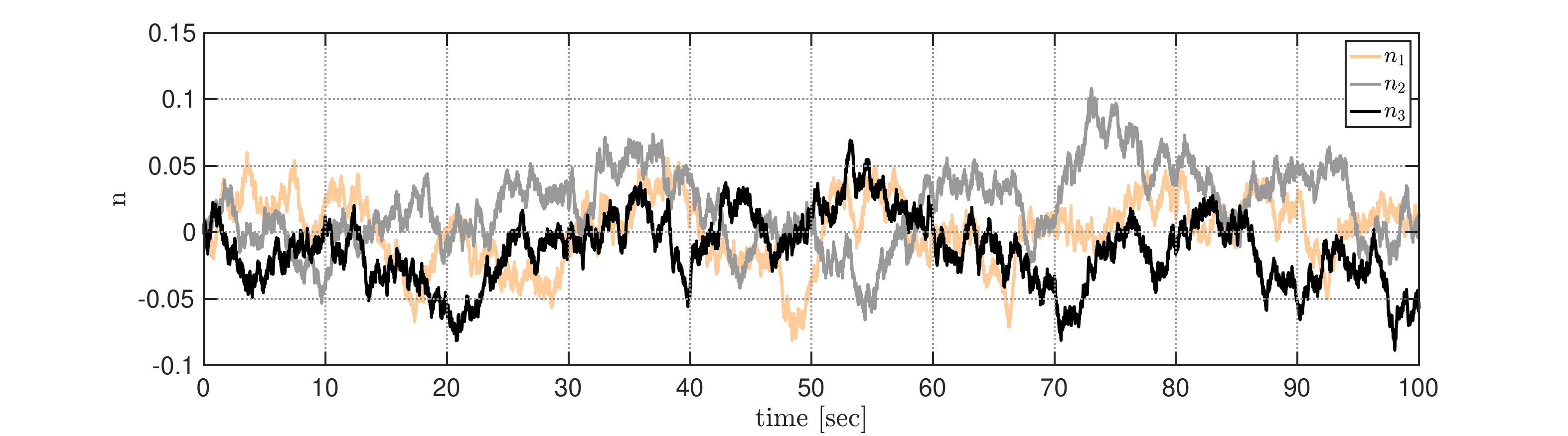}
    \caption{A realization of the Gaussian noise resulting from \eqref{eq:lowpass-noise}.}
    \label{fig:noise}
\end{figure}
%\end{comment}
% The robot begins from an initial state far from the trajectory to highlight its transient behavior. Fig.~\ref{fig:ramp} shows the results, with the same conventions for $\sigma=0$ and $\sigma=1$ as in Simulation~1.

The plots for both simulations (Figs.~\ref{fig:circle} and~\ref{fig:ramp}) demonstrate the controller's ability to solve the trajectory tracking problem effectively. When the switching signal $\sigma=1$, the system transitions into dexterous mode, where the smooth tracking of both $\yv_1$ and $\yv_2$ is ensured.
When the switching signal $\sigma=0$, the system transitions into energy saving mode, where
the velocity $v_2$ is brought exponentially to zero while ensuring smooth tracking of $\yv_1$.
It is noteworthy that while the trajectory $y_2^d(t)$ is defined at all times, it is only enforced when $\sigma=1$. When $\sigma=0$, the angle variable $\theta$ ceases to follow the prescribed trajectory. Instead, the robot moves in a unicycle-like fashion, with $\theta(t)$ evolving according to the platform’s flat outputs $x$ and $y$, i.e., $\theta(t) = \mathrm{atan2}\{\dot{y}, \dot{x}\}$. Despite this deviation in $\theta$, the position coordinates $x$ and $y$, and hence $\yv_1$, continue to perfectly track the desired trajectory.

The videos corresponding to the simulations of Figs.~\ref{fig:circle} and~\ref{fig:ramp} are available at \mbox{{\small\url{https://youtu.be/Wn4hVNXEjmc}}}.

\begin{comment}
    \section{Practical Remarks}

\begin{rem}
If one wants to include also input constraints in the problem, a practical possibility is to replace the implementation of the controller via matrix inversion in~\eqref{eq:control} with the solution of the following (convex) optimization problem 
\begin{equation}
   \begin{split}
       \uv^* = \argmin &||\vv - \bv(\xv)-\tilde{\Am}(\xv,\sigma)\uv||_\Wm ^2 \\
       \mathrm{s.t.}& \uv \in \mathcal{U} 
   \end{split} 
\end{equation}
with $\bv(\xv)$ as in \eqref{eq:b}, $\vv$ as in \eqref{eq:v}, $\tilde{\Am}(\xv,\sigma)$  as in \eqref{eq:Atilde} and \textit{eventually} a  positive-definite weight matrix $\Wm$.
\end{rem}
\end{comment}

\section{Conclusion}

This letter presents a novel control approach that selectively utilizes system inputs,  reducing energy consumption by eliminating unnecessary inputs for the primary task.
Beyond the primary control objective, a secondary task is dynamically assigned based on an input-switching signal.
This signal alternates between tracking additional system variables and driving energy-intense variables to zero. 
A key contribution of this work is demonstrating that the primary task is always fulfilled, irrespective of the switching signal's state. 

The proposed method provides the following advantages compared with state of the art alternatives:
 \begin{enumerate}
     \item Formal guarantees of \textit{exponential} stability of the primary objective with dynamics that are decoupled from the switching signal.
     \item Formal guarantees of \textit{exponential} stability of the secondary task while the switching signal is constant.
     \item A smooth control law with a closed-form expression.
 \end{enumerate}
Since the method is based on feedback linearization, a potential disadvantage is its lack of robustness to uncertainties in system model parameters. 
Additionally, the method requires a specific form of the system in order to be applied, which limits its generality.
Nevertheless, the domain of application of the method is  large and includes several practical systems, such as the Mecanum wheel robot presented in the example.

Future research will focus on robustifying and experimentally validating these theoretical findings,  particular to the Mecanum Wheel.

\balance
\bibliographystyle{IEEEtran}
\bibliography{bibAlias,bibAF,bibCustom,ref2}

\end{document}